\documentclass[11pt]{article}

\pdfoutput=1

\usepackage[numbers]{natbib}

\usepackage[T1]{fontenc}
\usepackage[latin9]{inputenc}
\usepackage{float}
\usepackage{fullpage,times,color,url}
\usepackage{boxedminipage}

\usepackage{amsmath,amsthm,amssymb}
\usepackage{mathtools}

\usepackage{graphicx}
\usepackage{ltxtable} 
\newcolumntype{L}{>{\centering\arraybackslash} m{0.04\columnwidth}} 
\newcolumntype{R}{>{\centering\arraybackslash} m{0.48\columnwidth}} 
\newcolumntype{S}{>{\centering\arraybackslash} m{0.32\columnwidth}} 

\usepackage{multirow}
\usepackage{array}

\usepackage{tikz}
\usepackage{pgfplots}

\newtheorem{lemma}{Lemma}

\newtheorem{theorem}{Theorem}

\newcommand{\E}{\mathbb{E}}

\newcommand{\reals}{\mathbb{R}}

\newcommand{\thmref}[1]{Theorem~\ref{#1}}

\newcommand{\lemref}[1]{Lemma~\ref{#1}}

\newenvironment{myalgo}[1]%
{
\begin{center}
\begin{boxedminipage}{0.8\linewidth}
\begin{center}
\textbf{\texttt{#1}}
\end{center}
\rm
\begin{tabbing}
....\=...\=...\=...\=...\=  \+ \kill
} %
{\end{tabbing} 
\end{boxedminipage} \end{center} 
}

\title{SDCA without Duality}
\author{Shai Shalev-Shwartz\thanks{School of Computer Science and
    Engineering, The Hebrew University, Jerusalem, Israel}}
\date{}

\begin{document}

\maketitle

\begin{abstract}
Stochastic Dual Coordinate Ascent is a popular method for
solving regularized loss minimization for the case of convex losses.
In this paper we show how a variant of SDCA can be applied for non-convex losses.
We prove linear convergence rate even if individual loss functions are
non-convex as long as the expected loss is convex. 
\end{abstract}

\section{Introduction}

The following regularized loss minimization problem is associated with many machine
learning methods:
\[
\min_{w\in \reals^d} P(w) := \frac{1}{n} \sum_{i=1}^n \phi_i(w) + \frac{\lambda}{2}
\|w\|^2 ~.
\]
One of the most popular methods for solving this problem is Stochastic
Dual Coordinate Ascent (SDCA). \cite{ShalevZh2013} analyzed this
method, and showed that when each $\phi_i$ is $L$-smooth and convex
then the convergence rate of SDCA is $\tilde{O}((L/\lambda +
n)\log(1/\epsilon))$. 

As its name indicates, SDCA is derived by considering a dual
problem. In this paper, we consider the possibility of applying SDCA
for problems in which individual $\phi_i$ are non-convex, e.g., deep
learning optimization problems. In many such cases, the dual problem
is meaningless. Instead of directly using the dual problem, we
describe and analyze a variant of SDCA in which only gradients of
$\phi_i$ are being used (similar to option 5 in the pseudo code of
Prox-SDCA given in \cite{ShalevZhangAcc2015}). Following
\cite{johnson2013accelerating}, we show that SDCA is a variant of the
Stochastic Gradient Descent (SGD), that is, its update is based on an
unbiased estimate of the gradient. But, unlike the vanilla SGD, for
SDCA the variance of the estimation of the gradient tends 
to zero as we converge to a minimum.

For the case in which each $\phi_i$ is $L$-smooth and convex, we
derive the same linear convergence rate of $\tilde{O}((L/\lambda +
n)\log(1/\epsilon))$ as in \cite{ShalevZh2013}, but with a simpler,
direct, dual-free, proof. We also provide a linear convergence rate for the
case in which individual $\phi_i$ can be non-convex, as long as the
average of $\phi_i$ are convex. The rate for non-convex losses has a
worst dependence on $L/\lambda$ and we leave it open to see if a
better rate can be obtained for the non-convex case. 


\paragraph{Related work:} 
In recent years, many methods for optimizing regularized loss
minimization problems have been proposed. For example,
SAG~\cite{LSB12-sgdexp}, SVRG~\cite{johnson2013accelerating},
Finito~\cite{defazio2014finito}, SAGA~\cite{defazio2014saga}, and
S2GD~\cite{konevcny2013semi}. The best convergence rate is for
accelerated SDCA~\cite{ShalevZhangAcc2015}.  A systematic study of the
convergence rate of the different methods under non-convex losses is
left to future work.


\section{SDCA without Duality}

We maintain pseudo-dual vectors $\alpha_1,\ldots,\alpha_n$, where each
$\alpha_i \in \reals^d$. 

\begin{myalgo}{Dual-Free SDCA($P,T,\eta,\alpha^{(0)}$)} 
\textbf{Goal:} Minimize $P(w) = \frac{1}{n} \sum_{i=1}^n
\phi_i(w) + \frac{\lambda}{2} \|w\|^2$ \\
\textbf{Input:} Objective $P$, number of iterations $T$, step size
$\eta$ s.t. $\beta := \eta \lambda n < 1$, \+ \\ initial dual vectors $\alpha^{(0)} =
(\alpha_1^{(0)},\ldots,\alpha_n^{(0)}$ \- \\ 
\textbf{Initialize:} $w^{(0)} = \frac{1}{\lambda n} \sum_{i=1}^n
\alpha_i^{(0)}$ \\
\textbf{For~} $t=1,\ldots,T$ \+ \\
Pick $i$ uniformly at random from $[n]$ \\
Update: $\alpha_{i}^{(t)} = \alpha_{i}^{(t-1)}  - \eta \lambda n
\left( \nabla \phi_i(w^{(t-1)}) + \alpha_{i}^{(t-1)} \right)$ \\
Update: $w^{(t)} = w^{(t-1)} - \eta \left( \nabla \phi_i(w^{(t-1)}) +  \alpha_{i}^{(t-1)} \right)$
\end{myalgo}
Observe that SDCA keeps the primal-dual relation
\[
w^{(t-1)} = \frac{1}{\lambda n} \sum_{i=1}^n \alpha^{(t-1)}_i
\]
Observe also that the update of $\alpha$ can be rewritten as
\[
\alpha_{i}^{(t)} = (1-\beta) \alpha_{i}^{(t-1)}  + \beta \left(
  -\nabla \phi_i(w^{(t-1)}) \right) ~,
\]
namely, the new value of $\alpha_i$ is a convex combination of its old
value and the negation of the gradient. 
Finally, observe that, conditioned on the value of $w^{(t-1)}$ and
$\alpha^{(t-1)}$, we have that
\begin{align*}
\E[w^{(t)} ] &=  w^{(t-1)} - \eta \left( \nabla \E \phi_i(w^{(t-1)}) +
  \E \alpha_{i}^{(t-1)} \right) \\
&=  w^{(t-1)}  - \eta \left( \nabla \frac{1}{n}
  \sum_{i=1}^n \phi_i(w^{(t-1)}) + \lambda w^{(t-1)} \right) \\
&= w^{(t-1)} - \eta \nabla P(w^{(t-1)}) ~.
\end{align*}
That is, SDCA is in fact an instance of Stochastic Gradient Descent. 
As we will see in the analysis section below, the advantage of SDCA
over a vanilla SGD algorithm is because the \emph{variance} of the
update goes to zero as we converge to an optimum. 

\section{Analysis}

The theorem below provides a linear convergence rate for smooth and convex
functions. The rate matches the analysis given in \cite{ShalevZh2013},
but the analysis is simpler and does not rely on duality. 
\begin{theorem} \label{thm:main}
Assume that each $\phi_i$ is $L$-smooth and convex, and the algorithm is run
with $\eta \le \frac{1}{L + \lambda n}$. 
Let $w^*$ be the minimizer of $P(w)$ and let $\alpha^*_i = -\nabla \phi_i(w^*)$.
Then, for every $t \ge 1$,
\[
\E\left[ \frac{\lambda}{2}
\|w^{(t)} - w^*\|^2 + \frac{1}{2Ln} \sum_{i=1}^n [ \|\alpha^{(t)}_i  -
\alpha^*_i\|^2] \right] ~\le~
e^{-\eta \lambda t} ~ \left[ \frac{\lambda}{2}
\|w^{(0)} - w^*\|^2 + \frac{1}{2Ln} \sum_{i=1}^n [ \|\alpha^{(0)}_i  -
\alpha^*_i\|^2] \right] ~.
\]
In particular, setting $\eta = \frac{1}{L + \lambda n} $, then after 
\[
T \ge 
\tilde{\Omega}\left( \frac{L}{\lambda} + n\right)
\]
iterations we will have  $\E[P(w^{(T)})-P(w^*)] \le \epsilon$.
\end{theorem}

The theorem below provides a linear convergence rate for smooth
functions, without assuming that individual $\phi_i$ are convex. We
only require that the average of $\phi_i$ is convex. The dependence on
$L/\lambda$ is worse in this case. 
\begin{theorem} \label{thm:main2} Assume that each $\phi_i$ is
  $L$-smooth and that the average function, $\frac{1}{n} \sum_{i=1}^n
  \phi_i$, is convex. Let $w^*$ be the minimizer of $P(w)$ and let
  $\alpha^*_i = -\nabla \phi_i(w^*)$. Then, if we run SDCA with $\eta
  = \min\{\frac{\lambda}{2L^2} ~,~ \frac{1}{2\lambda n}\}$, we have
  that
\[
\E\left[ \frac{\lambda}{2}
\|w^{(t)} - w^*\|^2 + \frac{\lambda}{2L^2 n} \sum_{i=1}^n [ \|\alpha^{(t)}_i  -
\alpha^*_i\|^2]  \right] ~\le~
e^{-\eta \lambda t} ~ \left[ \frac{\lambda}{2} 
\|w^{(0)} - w^*\|^2 + \frac{\lambda}{2L^2 n} \sum_{i=1}^n [ \|\alpha^{(0)}_i  -
\alpha^*_i\|^2] \right] ~.
\]
It follows that whenever 
\[
T \ge \tilde{\Omega}\left( \frac{L^2}{\lambda^2} + n\right)
\]
we have that $\E[P(w^{(T)})-P(w^*)] \le \epsilon$.
\end{theorem}

\subsection{SDCA as variance-reduced SGD}

As we have shown before, SDCA is an instance of SGD, in the sense that
the update can be written as $w^{(t)} =
w^{(t-1)} - \eta v_t$, with $v_t = \nabla \phi_i(w^{(t-1)}) +
\alpha_{i}^{(t-1)}$ satisfying $\E[v_t] = \nabla P(w^{(t-1)})$.

The advantage of SDCA over a generic SGD is that the variance of the
update goes to zero as we converge to the optimum. To see this,
observe that
\begin{align*}
\E[\|v_t\|^2] &= \E[\|\alpha_i^{(t-1)} + \nabla \phi_i(w^{(t-1)})\|^2] = \E[\|\alpha_i^{(t-1)}-\alpha^*_i+\alpha^*_i +\nabla \phi_i(w^{(t-1)})\|^2] \\
&\le 2\E[\|\alpha_i^{(t-1)} - \alpha_i^*\|^2] + 2\E[\|-\nabla
\phi_i(w^{(t-1)}) - \alpha_i^*\|^2]
\end{align*}
\thmref{thm:main} (or \thmref{thm:main2}) tells us that the term
$\E[\|\alpha_i^{(t-1)} - \alpha_i^*\|^2] $ goes to zero as $e^{-\eta
  \lambda t}$. For the second term, by smoothness of $\phi_i$ we have
$\|-\nabla \phi_i(w^{(t-1)}) - \alpha_i^*\| = \|\nabla
\phi_i(w^{(t-1)}) - \nabla \phi_i(w^*)\| \le L \|w^{(t-1)} - w^*\|$,
and therefore, using \thmref{thm:main} (or \thmref{thm:main2}) again,
the second term also goes to zero as $e^{-\eta \lambda t}$. All in
all, when $t \ge \tilde{\Omega}\left( \frac{1}{\eta \lambda}
  \log(1/\epsilon)\right)$ we will have that $\E[\|v_t\|^2] \le
\epsilon$.

\section{Proofs}

Observe that $0 = \nabla P(w^*) = \frac{1}{n} \sum_i \nabla
\phi_i(w^*) + \lambda w^*$, which implies that $w^* = \frac{1}{\lambda
  n} \sum_i \alpha_i^*$.

Define $u_i =-\nabla \phi_i(w^{(t-1)})$ and $v_t = -u_i +
\alpha_i^{(t-1)}$. We also denote two potentials:
\[
A_t = \frac{1}{n} \sum_{j=1}^n \|\alpha^{(t)}_j  - \alpha^*_j\|^2
~~~,~~~
B_t = \|w^{(t)}-w^*\|^2 ~.
\]
We will first analyze the evolution of $A_t$ and $B_t$.  If on round
$t$ we update using element $i$ then $\alpha_i^{(t)} = (1-\beta)
\alpha_i^{(t-1)} + \beta u_i$, where $\beta = \eta \lambda n$. It
follows that,
\begin{align}  \label{eqn:Aevo}
&A_t- A_{t-1} =  \frac{1}{n} \|\alpha^{(t)}_i  - \alpha^*_i\|^2- \frac{1}{n} \|\alpha^{(t-1)}_i  -
\alpha^*_i\|^2  \\ \nonumber
&= \frac{1}{n} \|(1-\beta) (\alpha^{(t-1)}_i  - \alpha^*_i) +
\beta(u_i - \alpha_i^*) \|^2- \frac{1}{n} \|\alpha^{(t-1)}_i  -
\alpha^*_i\|^2 \\ \nonumber
&= \frac{1}{n} \left( (1-\beta)\|\alpha^{(t-1)}_i  - \alpha^*_i\|^2 +
  \beta \|u_i - \alpha_i^*\|^2 -
  \beta(1-\beta)\|\alpha^{(t-1)}_i - u_i \|^2 - \|\alpha^{(t-1)}_i  -
\alpha^*_i\|^2 \right) \\ \nonumber
&= \frac{\beta}{n} \left( -\|\alpha^{(t-1)}_i  - \alpha^*_i\|^2 +
  \|u_i - \alpha_i^*\|^2 -
  (1-\beta) \|v_t\|^2 \right) \\ \nonumber
&= \eta \lambda \left( -\|\alpha^{(t-1)}_i  - \alpha^*_i\|^2 +
  \|u_i - \alpha_i^*\|^2 -
  (1-\beta) \|v_t\|^2 \right) ~.
\end{align}
In addition,
\begin{equation} \label{eqn:Bevo}
B_t-B_{t-1} = 
\|w^{(t)}-w^*\|^2-\|w^{(t-1)}-w^*\|^2 = -2\eta(w^{(t-1)}-w^*)^\top v_t + \eta^2 \|v_t\|^2 ~.
\end{equation}

The proofs of \thmref{thm:main} and \thmref{thm:main2} will follow by
studying different combinations of $A_t$ and $B_t$. 

\subsection{Proof of \thmref{thm:main2}}

Define 
\[
C_t = \frac{\lambda}{2} \left[ \frac{1}{L^2}
  A_t + B_t
\right] ~.
\]
Combining \eqref{eqn:Aevo} and \eqref{eqn:Bevo} we obtain
\begin{align*}
&C_{t-1} - C_t \\
&= \frac{\eta \lambda^2 }{2L^2} \left( \|\alpha^{(t-1)}_i  - \alpha^*_i\|^2 -
  \|u_i - \alpha_i^*\|^2 +
  (1-\beta) \|v_t\|^2 \right)+ \frac{\lambda}{2}\left[2\eta(w^{(t-1)}-w^*)^\top v_t - \eta^2 \|v_t\|^2\right]\\
  &= \eta \lambda\left[
  \frac{\lambda}{2 L^2}\left( \|\alpha^{(t-1)}_i  - \alpha^*_i\|^2 -
  \|u_i - \alpha_i^*\|^2 \right) + \left(\frac{\lambda (1-\beta)}{2L^2}-\frac{\eta}{2} \right) \|v_t\|^2 + (w^{(t-1)}-w^*)^\top v_t 
  \right] 
\end{align*}
The definition of $\eta$ implies that $\eta \le \lambda(1-\beta)/L^2$,
so the coefficient of $\|v_t\|^2$ is non-negative. By
smoothness of each $\phi_i$ we have $\|u_i - \alpha_i^*\|^2 = \|\nabla
\phi_i(w^{(t-1)})-\nabla \phi_i(w^*)\|^2 \le
L^2\|w^{(t-1)}-w^*\|^2$. Therefore,
\[
C_{t-1} - C_t ~\ge~ \eta \lambda\left[ \frac{\lambda}{2 L^2}
  \|\alpha^{(t-1)}_i  - \alpha^*_i\|^2 - \frac{\lambda}{2}
  \|w^{(t-1)}-w^*\|^2 +  (w^{(t-1)}-w^*)^\top v_t 
  \right] ~.
\]
Taking expectation of both sides (w.r.t. the choice of $i$ and conditioned on $w^{(t-1)}$ and
$\alpha^{(t-1)}$) and noting that $\E[v_t] = \nabla
P(w^{(t-1)})$, we obtain that
\[
\E[C_{t-1}-C_t] ~\ge~ \eta \lambda\left[ \frac{\lambda}{2 L^2}
  \E \|\alpha^{(t-1)}_i  - \alpha^*_i\|^2 - \frac{\lambda}{2}
  \|w^{(t-1)}-w^*\|^2 +  (w^{(t-1)}-w^*)^\top \nabla P(w^{(t-1)})   \right] ~.
\]
Using the  strong convexity of $P$ we have $ (w^{(t-1)}-w^*)^\top \nabla
P(w^{(t-1)}) \ge P(w^{(t-1)})-P(w^*) +
\frac{\lambda}{2} \| w^{(t-1)}-w^*\|^2$ and $P(w^{(t-1)})-P(w^*) \ge
\frac{\lambda}{2} \| w^{(t-1)}-w^*\|^2$, which together yields
$ (w^{(t-1)}-w^*)^\top  \nabla P(w^{(t-1)})\ge \lambda \| w^{(t-1)}-w^*\|^2 $. 
Therefore,
\begin{align*}
&E[C_{t-1} - C_t] \\
&\ge \eta \lambda\left[
  \frac{\lambda}{2L^2} \E \|\alpha^{(t-1)}_i  - \alpha^*_i\|^2 +
    \left(-\frac{\lambda L^2}{2L^2} + \lambda \right) 
 \| w^{(t-1)}-w^*\|^2  \right]  ~=~ \eta \lambda C_{t-1} ~.
\end{align*}
It follows that
\[
\E[C_t] \le (1-\eta \lambda) C_{t-1}
\]
and repeating this recursively we end up with
\[
\E[C_t] \le (1-\eta \lambda)^t C_0 \le e^{-\eta \lambda t} C_0 ~,
\]
which concludes the proof of the first part of \thmref{thm:main2}.
The second part follows by
observing that $P$ is $(L+\lambda)$ smooth, which gives $P(w)-P(w^*)
\le \frac{L+\lambda}{2} \|w-w^*\|^2$. 

\subsection{Proof of \thmref{thm:main}}

In the proof of \thmref{thm:main} we bounded the term
$\|u_i-\alpha_i^*\|^2$  by $L^2 \|w^{(t-1)}-w^*\|^2$ based on the
smoothness of $\phi_i$. We now assume that $\phi_i$ is also convex,
which enables to bound $\|u_i-\alpha_i^*\|^2$ based on the current sub-optimality. 
\begin{lemma} \label{lem:uabound}
Assume that each $\phi_i$ is $L$-smooth and convex. Then, for every
$w$, 
\[ \frac{1}{n} \sum_{i=1}^n \|\nabla \phi_i(w) - \nabla \phi_i(w^*)\|^2 \le 2L\left(P(w)-P(w^*) - \frac{\lambda}{2} \|w-w^*\|^2\right) ~.
\]
\end{lemma}
\begin{proof}
For every $i$, define
\[
g_i(w) = \phi_i(w) - \phi_i(w^*) - \nabla \phi_i(w^*)^\top (w-w^*) ~.
\]
Clearly, since $\phi_i$ is $L$-smooth so is $g_i$. In addition, by
convexity of $\phi_i$ we have $g_i(w) \ge 0$ for all $w$. It follows
that $g_i$ is non-negative and smooth, and therefore, it is
self-bounded (see Section 12.1.3 in \cite{MLbook}):
\[
\|\nabla g_i(w)\|^2 \le 2L g_i(w) ~.
\]
Using the definition of $g_i$, we obtain
\[
\|\nabla \phi_i(w) - \nabla \phi_i(w^*)\|^2 = \|\nabla g_i(w)\|^2 \le 2L g_i(w) = 2L\left[ \phi_i(w) - \phi_i(w^*) - \nabla \phi_i(w^*)^\top (w-w^*)\right] ~.
\]
Taking expectation over $i$ and observing that $P(w) = \E \phi_i(w) +
\frac{\lambda}{2} \|w\|^2$ and $0 = \nabla P(w^*) = \E \nabla \phi_i(w^*) + \lambda w^*$ we obtain
\begin{align*}
\E \|\nabla \phi_i(w) - \nabla \phi_i(w^*)\|^2 &\le 
2L\left[ P(w)- \frac{\lambda}{2} \|w\|^2 - P(w^*) + \frac{\lambda}{2}
  \|w^*\|^2 + \lambda w^{*\top} (w-w^*)\right] \\
&= 2L\left[ P(w) - P(w^*) - \frac{\lambda}{2} \|w-w^*\|^2 \right]~.
\end{align*}
\end{proof}

We now consider the potential
\[
D_t =  \frac{1}{2L} A_t + \frac{\lambda}{2} B_t ~.
\]
Combining \eqref{eqn:Aevo} and \eqref{eqn:Bevo} we obtain
\begin{align*}
&D_{t-1}-D_t \\
&= \frac{\eta \lambda}{2L} \left( \|\alpha^{(t-1)}_i  - \alpha^*_i\|^2 -
  \|u_i - \alpha_i^*\|^2 +
  (1-\beta) \|v_t\|^2 \right)+ \frac{\lambda}{2}\left[2\eta(w^{(t-1)}-w^*)^\top v_t - \eta^2 \|v_t\|^2\right]\\
  &= \eta \lambda\left[
  \frac{1}{2L}\left( \|\alpha^{(t-1)}_i  - \alpha^*_i\|^2 -
  \|u_i - \alpha_i^*\|^2 \right) + \left(\frac{(1-\beta)}{2L}-\frac{\eta}{2} \right) \|v_t\|^2 + (w^{(t-1)}-w^*)^\top v_t 
  \right] \\
  &\ge \eta \lambda\left[
  \frac{1}{2L}\left( \|\alpha^{(t-1)}_i  - \alpha^*_i\|^2 -
  \|u_i - \alpha_i^*\|^2 \right) + (w^{(t-1)}-w^*)^\top v_t 
  \right] ~,
\end{align*}
where in the last inequality we used the assumption 
\[
\eta \le \frac{1}{L+\lambda n} ~~\Rightarrow~~ \eta \le \frac{1-\beta}{L} ~.
\]
Take expectation of the above w.r.t. the choice of $i$, using \lemref{lem:uabound}, using $\E[v_t] = \nabla P(w^{(t-1)})$, and using  convexity of $P$ that yields $P(w^*)-P(w^{(t-1)})\ge (w^*-w^{(t-1)})^\top \nabla P(w^{(t-1)})$, we obtain
\begin{align*}
&\E[D_{t-1}-D_t] \\
&\ge \eta \lambda\left[
  \frac{1}{2L}\left( \E\|\alpha^{(t-1)}_i  - \alpha^*_i\|^2 -
  \E\|u_i - \alpha_i^*\|^2 \right) + (w^{(t-1)}-w^*)^\top \E v_t 
  \right] \\
&\ge \eta \lambda\left[
  \frac{1}{2L} \E\|\alpha^{(t-1)}_i  - \alpha^*_i\|^2 -
   \left(P(w^{(t-1)})-P(w^*) - \frac{\lambda}{2} \|w^{(t-1)}-w^*\|^2\right) + (w^{(t-1)}-w^*)^\top \nabla P(w^{(t-1)})
  \right] \\
  &\ge \eta \lambda\left[ \frac{1}{2L} \E\|\alpha^{(t-1)}_i  - \alpha^*_i\|^2 + \frac{\lambda}{2} \|w^{(t-1)}-w^*\|^2 \right] = -\eta \lambda D_{t-1}
\end{align*}
This gives $\E[D_t] \le (1-\eta \lambda) D_{t-1} \le e^{-\eta \lambda}
D_{t-1}$, which concludes the proof of the first part of the
theorem. The second part follows by
observing that $P$ is $(L+\lambda)$ smooth, which gives $P(w)-P(w^*)
\le \frac{L+\lambda}{2} \|w-w^*\|^2$.

\bibliographystyle{plainnat}
\bibliography{curRefs}

\begin{thebibliography}{8}
\providecommand{\natexlab}[1]{#1}
\providecommand{\url}[1]{\texttt{#1}}
\expandafter\ifx\csname urlstyle\endcsname\relax
  \providecommand{\doi}[1]{doi: #1}\else
  \providecommand{\doi}{doi: \begingroup \urlstyle{rm}\Url}\fi

\bibitem[Defazio et~al.(2014{\natexlab{a}})Defazio, Bach, and
  Lacoste-Julien]{defazio2014saga}
Aaron Defazio, Francis Bach, and Simon Lacoste-Julien.
\newblock Saga: A fast incremental gradient method with support for
  non-strongly convex composite objectives.
\newblock In \emph{Advances in Neural Information Processing Systems}, pages
  1646--1654, 2014{\natexlab{a}}.

\bibitem[Defazio et~al.(2014{\natexlab{b}})Defazio, Caetano, and
  Domke]{defazio2014finito}
Aaron~J Defazio, Tib{\'e}rio~S Caetano, and Justin Domke.
\newblock Finito: A faster, permutable incremental gradient method for big data
  problems.
\newblock \emph{arXiv preprint arXiv:1407.2710}, 2014{\natexlab{b}}.

\bibitem[Johnson and Zhang(2013)]{johnson2013accelerating}
Rie Johnson and Tong Zhang.
\newblock Accelerating stochastic gradient descent using predictive variance
  reduction.
\newblock In \emph{Advances in Neural Information Processing Systems}, pages
  315--323, 2013.

\bibitem[Kone{\v{c}}n{\`y} and Richt{\'a}rik(2013)]{konevcny2013semi}
Jakub Kone{\v{c}}n{\`y} and Peter Richt{\'a}rik.
\newblock Semi-stochastic gradient descent methods.
\newblock \emph{arXiv preprint arXiv:1312.1666}, 2013.

\bibitem[{Le Roux} et~al.(2012){Le Roux}, {Schmidt}, and {Bach}]{LSB12-sgdexp}
Nicolas {Le Roux}, Mark {Schmidt}, and Francis {Bach}.
\newblock A stochastic gradient method with an exponential convergence rate for
  finite training sets.
\newblock In \emph{Advances in Neural Information Processing Systems}, pages
  2663--2671, 2012.

\bibitem[Shalev-Shwartz and Zhang(2015)]{ShalevZhangAcc2015}
S.~Shalev-Shwartz and T.~Zhang.
\newblock Accelerated proximal stochastic dual coordinate ascent for
  regularized loss minimization.
\newblock \emph{Mathematical Programming SERIES A and B (to appear)}, 2015.

\bibitem[Shalev-Shwartz and Ben-David(2014)]{MLbook}
Shai Shalev-Shwartz and Shai Ben-David.
\newblock \emph{Understanding Machine Learning: From Theory to Algorithms}.
\newblock Cambridge university press, 2014.

\bibitem[Shalev-Shwartz and Zhang(2013)]{ShalevZh2013}
Shai Shalev-Shwartz and Tong Zhang.
\newblock Stochastic dual coordinate ascent methods for regularized loss
  minimization.
\newblock \emph{Journal of Machine Learning Research}, 14:\penalty0 567--599,
  Feb 2013.

\end{thebibliography}

\end{document}